\documentclass{article}

\usepackage{microtype}
\usepackage{graphicx}
\usepackage{booktabs} % for professional tables

\usepackage{hyperref}

\usepackage[accepted]{icml2023}

\usepackage{amsmath}
\usepackage{amssymb}
\usepackage{mathtools}
\usepackage{amsthm}

\usepackage{url}            % simple URL typesetting
\usepackage{amsfonts}       % blackboard math symbols
\usepackage{nicefrac}       % compact symbols for 1/2, etc.
\usepackage{xcolor}         % colors
\usepackage{amssymb}
\usepackage{enumitem}
\usepackage{graphicx}
\usepackage{subcaption}

\usepackage[capitalize,noabbrev]{cleveref}

\usepackage[textsize=tiny]{todonotes}

\usepackage{color}
\usepackage{dsfont}
\usepackage{multirow}
\usepackage{natbib}

\newtheorem{theorem}{Theorem}

\newcommand{\bbP}{\mathbb{P}}

\newcommand{\bw}{\mathbf{w}}
\newcommand{\be}{\mathbf{e}}

\newcommand{\bz}{\mathbf{z}}

\newcommand{\bp}{\mathbf{p}}

\newcommand{\calD}{\mathcal{D}}

\newcommand{\calM}{\mathcal{M}}

\DeclareMathOperator*{\argmax}{arg\,max}

\newcommand*{\imag}[1]{%
    \raisebox{-.2\baselineskip}{%
        \includegraphics[
        height=\baselineskip,
        width=\baselineskip,
        keepaspectratio,
        ]{#1}%
    }%
}

\newcommand*{\imagg}[1]{%
    \raisebox{-.1\baselineskip}{%
        \includegraphics[
        height=\baselineskip,
        width=\baselineskip,
        keepaspectratio,
        ]{#1}%
    }%
}

\begin{document}

\twocolumn[
\icmltitle{Analyzing Privacy Leakage in Machine Learning via Multiple Hypothesis Testing: A Lesson From Fano}



\icmlsetsymbol{equal}{*}

\begin{icmlauthorlist}
\icmlauthor{Chuan Guo}{meta}
\icmlauthor{Alexandre Sablayrolles}{meta}
\icmlauthor{Maziar Sanjabi}{meta}
\end{icmlauthorlist}

\icmlaffiliation{meta}{Meta AI}

\icmlcorrespondingauthor{Chuan Guo}{chuanguo@meta.com}

\icmlkeywords{Machine Learning, ICML}

\vskip 0.3in
]



\printAffiliationsAndNotice{} % otherwise use the standard text.

\begin{abstract}
Differential privacy (DP) is by far the most widely accepted framework for mitigating privacy risks in machine learning. However, exactly how small the privacy parameter $\epsilon$ needs to be to protect against certain privacy risks in practice is still not well-understood. In this work, we study data reconstruction attacks for discrete data and analyze it under the framework of multiple hypothesis testing.
For a learning algorithm satisfying $(\alpha, \epsilon)$-R\'{e}nyi DP, we utilize different variants of the celebrated Fano's inequality to upper bound the attack advantage of a data reconstruction adversary.
Our bound can be numerically computed to calibrate the privacy parameter $\epsilon$ to the desired level of privacy protection in practice, and complements the empirical evidence for the effectiveness of DP against data reconstruction attacks even at relatively large values of $\epsilon$.
%Importantly, we show that if the underlying private data takes values uniformly from a set of size $M$, then the target $\epsilon$ can be $O(\log M)$ before the adversary gains significant inferential power.
%Our analysis offers theoretical evidence for the empirical effectiveness of DP against data reconstruction attacks even at relatively large values of $\epsilon$.
\end{abstract}

\section{Introduction}
\label{sec:intro}

As machine learning becomes increasingly ubiquitous in the real world, proper understanding of the privacy risks of ML also becomes a crucial aspect for its safe adoption. Numerous prior works have demonstrated privacy vulnerabilities throughout the ML training pipeline~\citep{shokri2017membership, song2017machine, nasr2019comprehensive, zhu2019deep, carlini2021extracting}. So far the only comprehensive defense against privacy attacks is differential privacy (DP; \citet{dwork2006calibrating}), which has been successfully adapted for training private ML models~\citep{chaudhuri2011differentially, shokri2015privacy, abadi2016deep}. 

Unfortunately, differentially private training also comes at a huge cost to model accuracy if a small privacy parameter $\epsilon$ is desired.
%This is not unexpected since, in principle, some degree of memorization may be necessary for high-performance ML models~\citep{feldman2020does}.
In contrast, in terms of the level of empirical protection conferred by DP against privacy attacks, the picture is much more optimistic: Across a wide range of attacks including membership inference, attribute inference and data reconstruction, even a small amount of DP noise is sufficient for thwarting most attacks~\citep{jayaraman2019evaluating, zhu2019deep, carlini2019secret, hannun2021measuring}. However, there is very little theoretical understanding of this phenomenon.

In this paper, we analyze privacy leakage in connection to the multiple hypothesis testing problem in information theory, and show that the empirical privacy protection conferred by DP with high $\epsilon$ may be more than previously thought. To this end, we first define a game (see \autoref{fig:data_reconstruction}) between a private learner and an adversary that tries to perform a \emph{data reconstruction attack} against the learned model. We analyze this game using the celebrated Fano's inequality to derive upper bounds on the adversary's attack advantage when the model is trained differentially privately.
%The central tool in our analysis is Fano's inequality, which we utilize to bound the advantage of an adversary in terms of the mutual information between the private data and the trained model.

Our analysis reveals an interesting and practically important insight that the DP parameter $\epsilon$ should scale with the number of possible values $M$ that the private data can take on. When $M$ is large, \emph{e.g.}, $M=10^{10}$ when extracting social security numbers from a trained language model~\citep{carlini2019secret}, even a relatively large $\epsilon$ has sufficient protection against data reconstruction attacks. More generally, given an input data distribution and a DP parameter $\epsilon$, we give a numerical method for deriving an upper bound on the advantage of an arbitrary data reconstruction adversary. We empirically validate our bound against several existing attacks and show that it can provide useful guidance for selecting the appropriate value of $\epsilon$ in practice.

\begin{figure*}[t]
\centering
\includegraphics[width=.95\textwidth]{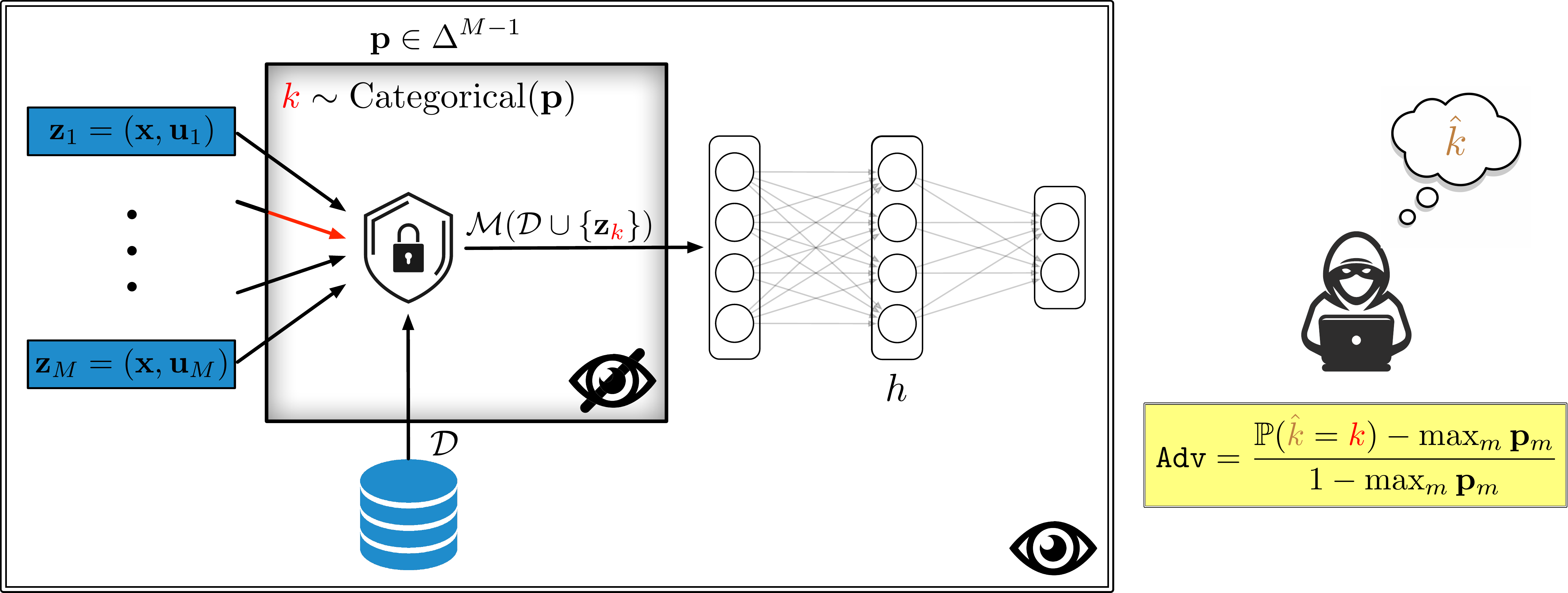}
\caption{Illustration of the data reconstruction attack game. The target data takes on $M$ possible values $\mathbf{u}_1,\ldots,\mathbf{u}_M$ with conditional probabilities $\mathbb{P}(\mathbf{u}_m | \mathbf{x}) = \mathbf{p}_m$. The game begins with ${\color{red} k}$ drawn from $\text{Categorical}(\mathbf{p})$, and the private learner \imag{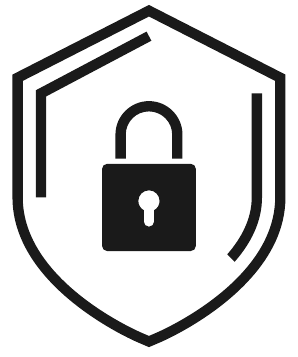} trains a model $h \leftarrow \mathcal{M}(\mathcal{D} \cup \{\mathbf{z}_{\color{red} k}\})$. The adversary guesses ${\color{brown} \hat{k}}$ and wins if ${\color{brown} \hat{k}} = {\color{red} k}$. Advantage is defined so that $\texttt{Adv} \leq 1$ and guessing ${\color{brown} \hat{k}} = \argmax_m \mathbf{p}_m$ achieves zero advantage. Processes inside the box marked with \imagg{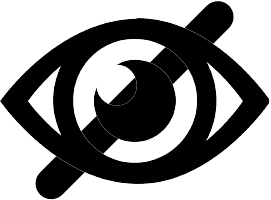} are unobserved by the adversary, while everything else is observable.}
\label{fig:data_reconstruction}
\end{figure*}

%We derive general bounds from R\'{e}nyi DP~\citep{mironov2017renyi}, as well as give specialized bounds for two common DP primitives---randomized response~\citep{warner1965randomized} and Gaussian mechanism~\citep{dwork2014algorithmic}---to upper bound the advantage of a data reconstruction adversary. Furthermore, we empirically validate these bounds against two attacks proposed in prior work: \textbf{1.} Attribute inference attack in pharmacogenetics modeling~\citep{fredrikson2014privacy}; and \textbf{2.} Canary extraction attack for language models~\citep{carlini2019secret}. Our evaluation gives concrete semantic privacy guarantees for DP mechanisms under practical settings, and validates the empirical finding that models trained using DP can achieve sufficient protection against data reconstruction attacks even when $\epsilon$ is high. As a result, we advocate for the use of our analysis for selecting the appropriate value of $\epsilon$ in practice.

\paragraph{Contributions.} Our main contributions are the following:
\begin{enumerate}[noitemsep, leftmargin=*, topsep=0pt]
    \item We formalize data reconstruction attacks for discrete data as an attack game (\autoref{sec:data_reconstruction}).
    \item We use Fano's inequality to derive a numerical method that upper bounds the adversary's advantage for $(\alpha, \epsilon)$-R\'{e}nyi DP mechanisms (\autoref{sec:fano} and \autoref{sec:bounds}).
    %\item We derive general upper bounds on mutual information for R\'{e}nyi DP mechanisms and specialize the analysis to two common DP primitives: randomized response~\citep{warner1965randomized} and Gaussian mechanism~\citep{dwork2014algorithmic} (\autoref{sec:bounds}).
    \item We experimentally validate our advantage bound against existing attack and show that it can be used to guide the selection of $\epsilon$ in practice (\autoref{sec:experiment}).
\end{enumerate}
\section{Background and Motivation}
\label{sec:background}

\paragraph{Privacy attacks.} The machine learning pipeline exposes training samples to the outside world through the training procedure and/or trained model. Prior works showed that adversaries can exploit this exposure to compromise the privacy of training samples. The most well-studied type of privacy attack is \emph{membership inference attack}~\citep{shokri2017membership, salem2018ml, yeom2018privacy, sablayrolles2019white}, which aims to infer whether a sample $\mathbf{z}$ corresponding to an individual's data was part of the model's training set. This membership status can be a very sensitive attribute, \emph{e.g.}, whether or not an individual participated in a cancer study indicates their disease status. Most state-of-the-art attacks~\citep{ye2021enhanced, watson2021importance, carlini2022membership} follow a common strategy of measuring the model's loss compared to that of a random model trained without the target sample $\mathbf{z}$, with a large difference indicating that the sample was seen during training (\emph{i.e.}, a member).

Other privacy attacks can extract more detailed information about a training sample beyond membership status. \emph{Attribute inference attacks}~\citep{fredrikson2014privacy, yeom2018privacy} aim to reconstruct a training sample given access to the trained model and partial knowledge of the sample. \emph{Data reconstruction attacks}~\citep{fredrikson2015model, carlini2019secret, carlini2021extracting, balle2022reconstructing} relax the partial knowledge assumption of attribute inference attacks and can recover training samples given only the trained model. In federated learning~\citep{mcmahan2017communication}, adversaries that observe the gradient updates can reconstruct private training samples using a process called \emph{gradient inversion attack}~\citep{zhu2019deep, geiping2020inverting}. The existence of these privacy attacks calls for countermeasures that can preserve the utility of ML models while preventing unintended leakage of private information.

\textbf{Differential privacy}~\citep{dwork2006calibrating} is a mathematical definition of privacy that upper bounds the amount of information leakage through a private mechanism. In the context of ML, the private mechanism $\mathcal{M}$ is a learning algorithm that, given any pair of datasets $\mathcal{D}$ and $\mathcal{D}'$ that differ in a single training sample, ascertains that $\mathcal{M}(\mathcal{D})$ and $\mathcal{M}(\mathcal{D}')$ are $\epsilon$-close in distribution for some chosen privacy parameter $\epsilon > 0$. In the classical definition of differential privacy, $\epsilon$-closeness is defined in terms of \emph{max divergence}: $\mathcal{M}$ is $\epsilon$-differentially private (denoted $\epsilon$-DP) if:
\begin{align*}
    D&_\infty(\mathcal{M}(\mathcal{D}) || \mathcal{M}(\mathcal{D}')) := \nonumber \\
    &\sup_O \left[ \log \mathbb{P}(\mathcal{M}(\mathcal{D}) \in O) - \log \mathbb{P}(\mathcal{M}(\mathcal{D}') \in O) \right] \leq \epsilon,
\end{align*}
where $O$ denotes a subset of the model space. One variant of DP that uses R\'{e}nyi divergence to quantify closeness is \emph{R\'{e}nyi DP} (RDP; \citet{mironov2017renyi}). For a given $\alpha \geq 1$, we say that $\calM$ is $(\alpha,\epsilon)$-RDP if:
\begin{align*}
    D_\alpha(&\mathcal{M}(\mathcal{D}) || \mathcal{M}(\mathcal{D}')) \nonumber \\
    & := \frac{1}{\alpha-1} \log \mathbb{E}_{h \sim \mathcal{M}(\mathcal{D}')} \left[ \frac{\mathbb{P}(\mathcal{M}(\mathcal{D}))^\alpha}{\mathbb{P}(\mathcal{M}(\mathcal{D}'))^\alpha} \right] \leq \epsilon.
\end{align*}
Notably, as $\alpha \rightarrow \infty$, $(\alpha, \epsilon)$-RDP coincides with $\epsilon$-DP.
%Variants that use other measures of information include concentrated DP~\cite{dwork2016concentrated, bun2016concentrated} and Gaussian DP~\cite{dong2019gaussian}.

\paragraph{Semantic privacy.} Differential privacy has been shown to effectively prevent all of the aforementioned privacy attacks when the privacy parameter $\epsilon$ is small enough~\citep{jayaraman2019evaluating, zhu2019deep, carlini2019secret, hannun2021measuring}. Conceptually, when $\epsilon \approx 0$, the learning algorithm outputs roughly the same distribution of models when a single training sample $\mathbf{z}$ is removed/replaced, hence an adversary cannot accurately infer any private information about $\mathbf{z}$. However, this reasoning does not quantify \emph{how small $\epsilon$ needs to be to prevent a certain class of privacy attacks to a certain degree}. In practice, this form of \emph{semantic guarantee} is arguably more meaningful, as it may inform policy decisions regarding the suitable range of $\epsilon$ to provide sufficient privacy protection, and enables privacy auditing~\citep{jagielski2020auditing} to verify that the learning algorithm's implementation is compliant.

Several existing works made partial progress towards answering this question. \citet{yeom2018privacy} formalized membership inference attacks by defining a game between the private learner and an adversary, and showed that the adversary's \emph{advantage}---how well the adversary can infer a particular sample's membership status---is bounded by $e^\epsilon-1$ when the learning algorithm is $\epsilon$-DP. This bound has been tightened significantly in subsequent works~\citep{erlingsson2019we, humphries2020differentially, mahloujifar2022optimal, thudi2022bounding}. Similarly, \citet{bhowmick2018protection, balle2022reconstructing, guo2022bounding} formalized data reconstruction attacks and showed that for DP learning algorithm, the adversary's expected reconstruction error can be lower bounded using DP, R\'{e}nyi-DP~\citep{mironov2017renyi}, and Fisher information leakage~\citep{hannun2021measuring} privacy accounting. Our work makes further progress in this direction by analyzing data reconstruction attacks using tools from the multiple hypothesis testing literature, which we show is well-suited for discrete data.
\section{Formalizing Data Reconstruction}
\label{sec:data_reconstruction}

To understand the semantic privacy guarantee for DP mechanisms against data reconstruction attacks, we first formally define a data reconstruction game for discrete data. Our formulation generalizes the membership inference game in existing literature~\citep{yeom2018privacy, humphries2020differentially}, while specializing the formulation of \citet{balle2022reconstructing} to discrete data.

\paragraph{Data reconstruction game.} Let $\mathcal{D}_\text{train} = \mathcal{D} \cup \{\mathbf{z}\}$ be the training set consisting of a \emph{public set} $\mathcal{D}$ and a \emph{private record} $\mathbf{z} = (\mathbf{x}, \mathbf{u})$, where $\mathbf{x}$ are attributes known to the adversary and $\mathbf{u}$ is unknown.
Let $\mathcal{M}$ be the learning algorithm. We consider a white-box adversary with full knowledge of the public set $\mathcal{D}$ and the trained model $h = \mathcal{M}(\mathcal{D}_\text{train})$ whose objective is to infer the unknown attributes $\mathbf{u}$. Importantly, we assume that the unknown attribute is discrete (\emph{e.g.}, gender, race, marital status) and can take on $M$ values $\mathbf{u}_1,\ldots,\mathbf{u}_M$. For example, the experiment setting of \citet{carlini2019secret} can be stated as $\mathbf{x} = ``\texttt{My social security number is}"$ and $\mathbf{u} \in \{0,1,\ldots,9\}^{10}$ is the SSN number.

The attack game (see \autoref{fig:data_reconstruction} for an illustration) begins by drawing a random index $k$ from a categorical distribution defined by the probability vector $\bp$ and setting the unknown attribute $\mathbf{u} = \mathbf{u}_k$. The private learner $\mathcal{M}$ then trains a model $h$ with $\mathbf{z} = \mathbf{z}_k = (\mathbf{x}, \mathbf{u}_k)$ and gives it to the adversary, who then outputs a guess $\hat{k}$ of the underlying index $k$. Note that both the random index $k$ and any randomness in the learning algorithm $\mathcal{M}$ are unobserved by the adversary, but the learning algorithm itself is known.

%We consider a generalization of the membership inference attack (MIA) game defined in Algorithm 1 of \cite{humphries2020differentially}. In Figure~\ref{fig:attribute_inference}, the private learner is given a fixed dataset $\calD$ and $M$ samples $\bz_1,\ldots,\bz_M \in \calZ$. A random index $k$ is generated from a categorical distribution defined by the probability vector $\bp$ and the model $h$ is trained using $\calA$ on the training set $\calD_\text{train} = \calD \cup \{\bz_k\}$. The adversary is then tasked with outputting a guess $\hat{k}$ of the underlying index $k$.

\paragraph{Success metric.} We generalize the \emph{advantage} metric~\citep{yeom2018privacy} used in membership inference attack games to multiple categories. Here, the (Bayes) optimal guessing strategy without observing $h$ is to simply guess $\hat{k} = \argmax_m \bp_m$ with success rate $\max_m \bp_m$. The probability of successfully guessing $k$ upon observing $h$, \emph{i.e.}, $\mathbb{P}(\hat{k} = k)$, must be at least $\max_m \bp_m$ in order to meaningfully leverage the private information contained in $h$ about $\mathbf{u}$. Thus, we define advantage as the (normalized) difference between $\mathbb{P}(\hat{k} = k)$ and the \emph{baseline success rate} $p^* := \max_m \bp_m$, \emph{i.e.},
\begin{equation}
    \label{eq:advantage}
    \texttt{Adv} := \frac{\mathbb{P}(\hat{k} = k) - p^*}{1 - p^*} \in [0,1].
\end{equation}

\paragraph{Interpretation.} Our data reconstruction game has the following important implications for privacy semantics.

1. \emph{The private attribute $\mathbf{u}$ is considered leaked if and only if it is guessed exactly.} This is a direct consequence of defining adversary success as $\hat{k} = k$. For example, if the attribute is a person's age, then guessing $\hat{k} = 50$ when the ground truth is $k=49$ should be considered more successful than when $k=40$. In such settings, it may be more suitable to partition the input space into broader categories, \emph{e.g.}, age ranges 0-9, 10-19, \emph{etc.}, to allow inexact guesses.

2. \emph{The attack game subsumes attribute inference attacks.} This can be done by setting the known attribute $\mathbf{x}$ accordingly. When $\mathbf{x} = \emptyset$, our game corresponds to the scenario commonly referred to as \emph{data reconstruction} in existing literature~\citep{carlini2019secret, carlini2021extracting, balle2022reconstructing}.

3. \emph{Prior information is captured through the sampling probability $\bp$.} Success rate of the Bayes optimal strategy is $p^* = \max_m \bp_m$, which depends on the sampling probability vector $\bp$. In the extreme case where $\bp$ is a delta distribution on some $k^*$, which corresponds to the adversary having perfect knowledge of the private attribute $\mathbf{u}$, the model $h$ provides \emph{no additional information} about $\mathbf{u}$. This is in accordance with the ``no free lunch theorem'' in differential privacy, which says that inference of private information is not preventable when the adversary is given arbitrary prior information~\citep{dwork2010difficulties}.

4. \emph{Advantage captures attacker uncertainty.} One can also interpret advantage as a measure of uncertainty for the recovered value $\hat{k}$. Even when advantage is low, it is possible that the adversary recovers $k$ by chance. For instance, an attack that always guesses a fixed $\hat{k}$ can correctly recover the true input $k$ with non-trivial probability, but this should not be considered a successful recovery. By defining advantage in this manner, we consider an attack successful only when it can reduce its uncertainty sufficiently below the inherent randomness in $k$.

\paragraph{Relation to prior work.} Our attack game is a direct generalization of the membership inference attack game in \citet{yeom2018privacy}. In particular, the ``strong adversary'' game defined in \citet{humphries2020differentially, nasr2021adversary} is a special case of our data reconstruction game with $M=2$. 
Our game formulation also complements \citet{guo2022bounding}, which bounds the inferential power of a data reconstruction attack under the assumption that semantic similarity between private records is measured by mean squared error (MSE), which is better suited to continuous-valued data. In contrast, our formulation considers discrete data with the zero-one similarity measure, which is better suited for categorical data such as gender, race and marital status.

Our definition can also be viewed as a specialization of the data reconstruction attack game defined in \citet{balle2022reconstructing}, where the learning algorithm $\mathcal{M}$ is said to be $(\eta, \gamma)$-ReRo (reconstruction robust) if:
\begin{equation}
    \label{eq:rero}
    \mathbb{P}_{k \sim \text{Categorical}(\mathbf{p}), h \sim \mathcal{M}(\mathcal{D} \cup \{\mathbf{z}_k\})}(\ell(\mathbf{z}_k, \hat{\mathbf{z}}) \leq \eta) \leq \gamma,
\end{equation}
where $\hat{\mathbf{z}}$ is the adversary's reconstruction of $\mathbf{z}_k$ and $\ell$ is an error function that measures how faithful the reconstruction $\hat{\mathbf{z}}$ is to $\mathbf{z}_k$. By setting $\ell(\mathbf{z}_k, \hat{\mathbf{z}}) = \mathds{1}(\hat{\mathbf{z}} = \mathbf{z}_k)$ and $\eta = 0$, one can show that \autoref{eq:rero} is satisfied if and only if $\mathbb{P}(\hat{k} = k) \leq \gamma$, thus if $\mathcal{M}$ is $(\eta,\gamma)$-ReRo with $\eta<1$ then $\texttt{Adv} \leq (\gamma - p^*) / (1 - p^*)$.
\section{Connecting Data Reconstruction to Multiple Hypothesis Testing}
\label{sec:fano}

In this section, we analyze the data reconstruction game defined in \autoref{sec:data_reconstruction} and derive a numerical method for the advantage upper bound when $\mathcal{M}$ is differentially private. 

\paragraph{Fano's inequality.} The central tool in our analysis is Fano's inequality---a celebrated result in information theory that lower bounds the error rate in multiple hypothesis testing. Suppose that $X$ is a discrete random variable taking on $M$ different values, and $Y$ is an observed random variable determined by $X$, \emph{e.g.}, $Y$ is the output of the private learning algorithm $\calM$. Let $\hat{X}$ be an estimate of $X$, and define the error random variable $E = (\hat{X} \neq X) \in \{0,1\}$. Then:
\begin{equation}
    \label{eq:fano}
    H(X | Y) \leq H(E) + \bbP(E = 1) \log(M-1),
\end{equation}
where $H(E)$ denotes the entropy of $E$ and $H(X|Y)$ denotes the conditional entropy of $X$ given $Y$.

%We can translate Fano's inequality to the attribute inference setting by letting $X = k$ and $Y = h$. Moreover, a special case of Fano's inequality for uniformly distributed $X$ states that $\bbP(E = 0) \leq (I(X; Y) + \log 2) / \log M$, where $I(X;Y)$ is the mutual information between $X$ and $Y$. Hence if $\bp_m = 1/M$ for all $m$, then
%\begin{equation}
%    \texttt{Adv} = 1 - \frac{M}{M-1} \bbP(\hat{k} \neq k) \leq 1 - \frac{M}{M-1} \left( 1 - \frac{I(k; h) + \log 2}{\log M} \right).
%\end{equation}
%The mutual information $I(k;h)$ has the following form\footnote{See section 2.3 of \url{https://people.ece.cornell.edu/acharya/teaching/ece6980f17/lectures/lec5.pdf}}:
%\begin{equation}
%    \label{eq:mutual_information}
%    I(k;h) \leq \frac{1}{M^2} \sum_{i=1}^M \sum_{j=1}^M D_1(\bbP(h~|~k=i)~||~\bbP(h~|~k=j)),
%\end{equation}
%which is the average KL divergence (or equivalently, Rényi divergence with order $\alpha=1$) between the distribution of models when the underlying training sample is $\bz_i$ vs. $\bz_j$. Thus if the learning algorithm $\calA$ is $\alpha=1$ Rényi-DP then the attribute inference advantage can be effectively controlled.

\paragraph{From Fano to advantage bound.} The attack game defined in \autoref{fig:data_reconstruction} essentially tasks the adversary with a \emph{multiple hypothesis test}, where the hypotheses are $H_k: h \sim \calM(\calD \cup \{\bz_k\})$ for $k=1,\ldots,M$. When $M=2$, we recover the binary hypothesis testing interpretation of membership inference attack~\citep{yeom2018privacy, humphries2020differentially}. Consequently, we can apply Fano's inequality to bound the adversary's advantage (\emph{cf.} Equation \ref{eq:advantage}) by setting $X = \text{Categorical}(\bp)$ and $Y = \mathcal{M}(\mathcal{D}_\text{train})$. In effect, $X$ encodes all the information contained in the sensitive attribute $\mathbf{u}$, and $Y$ is the output of a private mechanism that the adversary observes with partial information about $X$. For any adversary, let $t = \mathbb{P}(E=1) = \mathbb{P}(\hat{k} \neq k)$ be the error probability, and define:
\begin{equation}
    \label{eq:fano_equation}
    f(t) = H(X|Y) + t \log t + (1-t) \log (1-t) - t \log(M-1).
\end{equation}
Since $t = \mathbb{P}(E=1)$, Fano's inequality (\emph{cf.} \autoref{eq:fano}) defines a constraint $f(t) \leq 0$ on $t$. In other words, the error probability of \emph{any} adversary must satisfy $f(t) \leq 0$, and thus we can upper bound advantage (or equivalently, lower bound error probability) by solving
\begin{equation}
    \label{eq:fano_problem}
    t^* = \min \{t \in [0,1]: f(t) \leq 0\},
\end{equation}
and the advantage bound is $\texttt{Adv} \leq (1 - t^* - p^*) / (1 - p^*)$.
%Taking the derivative gives:
%\begin{equation}
%    f'(t) = \log t + 1 - \log(1-t) - 1 - \log(M-1) = \log t - \log(1-t) - \log(M-1).
%\end{equation}
%It is easy to see that $f'(t) \geq 0$ if and only if $t \geq 1 - 1/M$, so $f$ is decreasing for $t \in [0, 1-1/M)$ and increasing for $t \in (1-1/M, 1]$. Plugging in $t = 1 - 1/M$ gives:
%\begin{equation}
%    f(t) = H(X|Y) + (1-1/M) \log(M-1) - \log M - (1-1/M) \log(M-1) = H(X|Y) - \log M \leq 0,
%\end{equation}
%since $H(X|Y) \leq H(X)$ and $X$ takes at most $M$ distinct values, with equality if and only if $X|Y$ is the uniform distribution. Thus Fano's inequality (\emph{cf.} Equation \ref{eq:fano}) is satisfied for some $t = \bbP(E=1) \in [0, 1-1/M]$. When $X|Y$ is the uniform distribution, the only feasible solution is $t=1-1/M$, hence $\bbP(E=1) = 1-1/M$ and $\adv \leq 0$.

\paragraph{Mutual information privacy.} In order to compute the advantage bound, we need to first evaluate the conditional entropy term $H(X|Y)$ in \autoref{eq:fano_equation}. An alternative expression for conditional entropy is
$$H(X|Y) = H(X) - I(X;Y),$$
where $I(X;Y)$ is the mutual information between $X$ and $Y$. Since $H(X)$ depends only on the sampling probability vector $\bp$, the only factor that the DP mechanism controls is $I(X;Y)$. Intuitively, $I(X;Y)$ measures the amount of information leaked about the data $\mathbf{z}_k$ through the trained model, and a DP learning algorithm can limit this leakage to reduce the advantage of a data reconstruction adversary. It is easy to show that using mutual information to quantify privacy satisfies key properties of DP such as adaptive composition and post-processing.

The connection between mutual information and differential privacy has been observed in prior work on mutual information DP (MI-DP; \citet{mir2013information, cuff2016differential, wang2016relation}), quantitative information flow~\cite{alvim2011relation}, and privacy funnel~\cite{makhdoumi2014information, salamatian2020privacy}. Some works also empirically control mutual information in order to reduce privacy leakage~\citep{bertran2019adversarially, wang2021improving}. Similar to our approach, \citet{salamatian2020privacy} also used Fano's inequality to derive a closed-form lower bound on the error rate $\mathbb{P}(\hat{k} \neq k)$, but the bound is too loose for practical purposes. In contrast, we give a \emph{numerical method} below for upper bounding the adversary's advantage that is much tighter in practice.

\paragraph{Computing advantage bound.}  The advantage bound is a function of $t^*$ in \autoref{eq:fano_problem}, which we can obtain by solving a one-dimensional minimization problem as follows. Assume that the mechanism $\mathcal{M}$ satisfies $I(X;Y) \leq \mu$ for some $\mu > 0$, and define:
\begin{equation}
    \label{eq:fano_equation_eps}
    f_\mu(t) = H(X) - \mu + t \log t + (1-t) \log(1-t) - t \log(M-1).
\end{equation}
Since $H(X|Y) = H(X) - I(X;Y) \geq H(X) - \mu$, we have that $f_\mu(t^*) \leq f(t^*) \leq 0$. Thus, we can instead minimize $t$ subject to $f_\mu(t) \leq 0$ to obtain a lower bound $t^* \leq \min \{t \in [0,1] : f_\mu(t) \leq 0\}$ for the error probability of \emph{any} data reconstruction adversary.
Taking derivative gives:
\begin{equation}
    f_\mu'(t) = \log t - \log(1-t) - \log(M-1).
\end{equation}
It is easy to see that $f_\mu'(t) \geq 0$ if and only if $t \geq 1 - 1/M$, so $f_\mu$ is decreasing for $t \in [0, 1-1/M)$ and increasing for $t \in (1-1/M, 1]$. Moreover, we know that $f_\mu(0) > 0$ if $\mu < H(X)$, and $f_\mu(1-1/M) = H(X) - \mu - \log(M) \leq 0$. Hence we can use binary search over the interval $[0,1-1/M]$ to find $\min \{t \in [0,1]: f_\mu(t) \leq 0\}$. Algorithm \ref{alg:advantage_bound} summarizes the computation in pseudo-code.

\begin{algorithm}[h]
\caption{Computing the advantage bound using Fano's inequality.}
\label{alg:advantage_bound}
\begin{algorithmic}[1]
\STATE \textbf{Input}: Upper bound $I(X;Y) \leq \mu$, number of distinct attributes $M$, sampling probability vector $\bp$.
\STATE $p^* \gets \max_m \bp_m$
\IF{$\mu \geq H(\bp)$}
\STATE \textbf{return} $1$
\ELSE
\STATE Define $f_\mu(t) = H(\bp) - \mu + t \log t + (1-t) \log(1-t) - t \log(M-1)$.
\STATE \textbf{assert} $f_\mu(0) > 0$ and $f_\mu(1-1/M) \leq 0$
\STATE Use binary search to find $t^* = \min \{t \in [0,1-1/M]: f_\mu(t) \leq 0\}$.
\STATE \textbf{return} $(1-t^*-p^*) / (1-p^*)$
\ENDIF
\end{algorithmic}
\end{algorithm}

\section{Advantage Bounds for DP Mechanisms}
\label{sec:bounds}

The advantage bound in \autoref{sec:fano} depends crucially on the upper bound $I(X;Y) \leq \mu$ for the mutual information between the private attribute and the trained model. We first present a general bound when the mechanism is $(\alpha,\epsilon)$-RDP. For commonly used DP primitives such as randomized response and Gaussian mechanism, we show that $I(X;Y)$ can be exactly computed or estimated, leading to even tighter advantage bounds in practice.

\subsection{Bound from RDP}

The general bound for $(\alpha,\epsilon)$-RDP mechanisms involves a generalization of mutual information known as \emph{Arimoto information}~\citep{arimoto1977information, verdu2015alpha}, denoted $I_\alpha(X;Y)$.
Notably, Arimoto information has natural connections to the order-$\alpha$ R\'{e}nyi divergence used to define RDP, and coincides with mutual information when $\alpha = 1$. \autoref{thm:renyi_bound} below shows that if the private mechanism $\calM$ is $(\alpha,\epsilon)$-RDP then $I_\alpha(X;Y) \leq \epsilon$. The special case of $\alpha=1$ corresponds to $I(X;Y) = I_\alpha(X;Y) \leq \epsilon$, which we can use in Algorithm \ref{alg:advantage_bound} to derive an advantage bound. Proof is given in \autoref{sec:proofs}.

\begin{theorem}
\label{thm:renyi_bound}
For any $\alpha \geq 1$, if $\mathcal{M}$ is $(\alpha,\epsilon)$-RDP then $I_\alpha(X;Y) \leq \epsilon$.
\end{theorem}

\paragraph{Generalized Fano's inequality.} For $(\alpha,\epsilon)$-RDP mechanisms with $\alpha > 1$, one can apply \autoref{thm:renyi_bound} by simply leveraging the monotonicity of R\'{e}nyi divergence with respect to the order $\alpha$. That is, if $\calM$ is $(\alpha,\epsilon)$-RDP then it is also $(1,\epsilon)$-RDP, so $I(X;Y) \leq \epsilon$. However, this will give a loose advantage bound since the full range of $\alpha$ is not being effectively utilized. Fortunately, Fano's inequality can be generalized accordingly using Arimoto information for any order $\alpha$, which enables us to generalize the advantage bound in \autoref{sec:fano} with improved tightness. We give details for this generalized analysis in \autoref{sec:method_details}.

\begin{figure}[t]
\centering
\includegraphics[width=\linewidth]{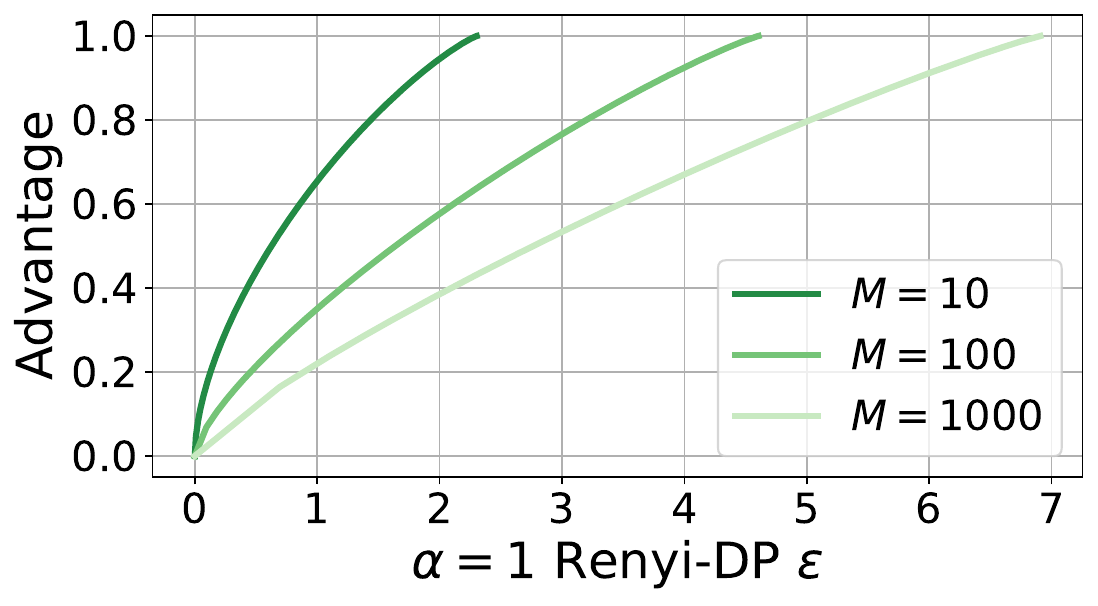}
\caption{Reconstruction advantage vs. $\alpha=1$ R\'{e}nyi DP parameter $\epsilon$ for different values of $M$. As the number of possibilities $M$ for the private attribute increases, the advantage curve becomes flatter, suggesting that a high $\epsilon$ still gives strong protection against data reconstruction.}
\label{fig:advantage_vs_eps}
\end{figure}

\paragraph{Example.} We showcase the power of the Fano's inequality analysis in combination with \autoref{thm:renyi_bound} in the following example. Here, we consider $M$ training samples $\bz_1,\ldots,\bz_M$ with uniform sampling probability $\bp_m = 1/M$ for $m=1,\ldots,M$, hence $H(X) = \log M$. For a $(1,\epsilon)$-RDP learning algorithm $\mathcal{M}$, we can use \autoref{thm:renyi_bound} to upper bound $I(X;Y)$ and use Algorithm \ref{alg:advantage_bound} to compute $t^* = \mathbb{P}(E=1)$ and the advantage bound accordingly. \autoref{fig:advantage_vs_eps} shows the advantage bound as a function of $\epsilon$ for different values of $M$. As expected, all curves are monotonically increasing in $\epsilon$, \emph{i.e.}, more privacy leakage results in higher advantage for the adversary.

More interestingly, for the same value of $\epsilon$, having a higher $M$ reduces advantage, which reflects the fact that there is more entropy in the private attribute and thus the adversary's success probability decreases with $M$. In fact, it is easy to show that when $M$ is uniformly distributed, Fano's inequality simplifies to:
\begin{align*}
&\hspace{3ex} H(X) - I(X;Y) \leq H(E) + \mathbb{P}(E=1) \log(M-1) \\
&\Leftrightarrow I(X;Y) \geq (1 - \mathbb{P}(E=1)) \log M - \log 2.
\end{align*}
Thus, for a given $c > 0$, the adversary's advantage at $\epsilon \approx c \log M$ is the same for all values of $M$, suggesting that $\epsilon$ should scale with $\log M$ to target a specific level of privacy risk. This theoretical analysis confirms the empirical observation that even a large $\epsilon$ confers non-trivial protection against data reconstruction attacks~\citep{carlini2019secret}. For instance, the membership inference advantage at $\epsilon=1$ is the same as the data reconstruction advantage at $\epsilon = \log M \approx 23$ when the private attribute is uniformly randomly sampled from a set of size $M=10^{10}$.

\begin{figure*}[t]
    \centering
    \begin{subfigure}{.49\textwidth}
      \centering
      \includegraphics[width=.51\linewidth]{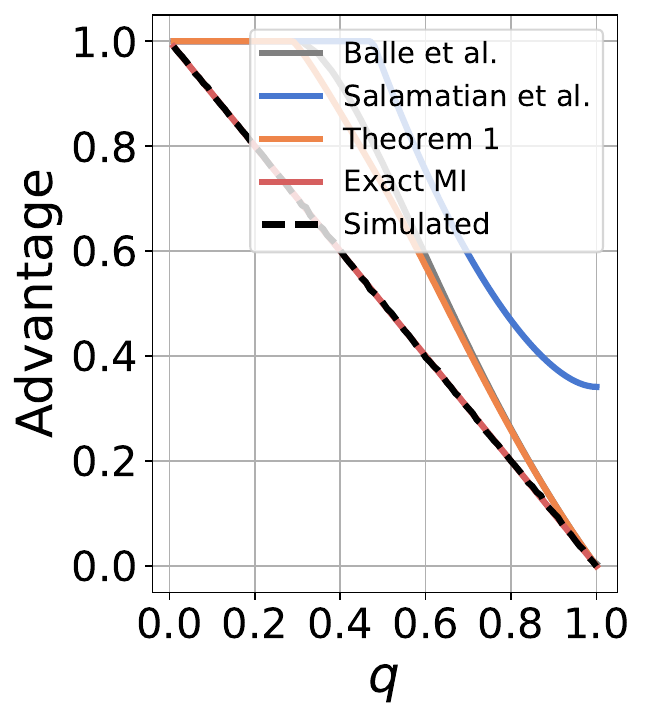}
      \includegraphics[width=.48\linewidth]{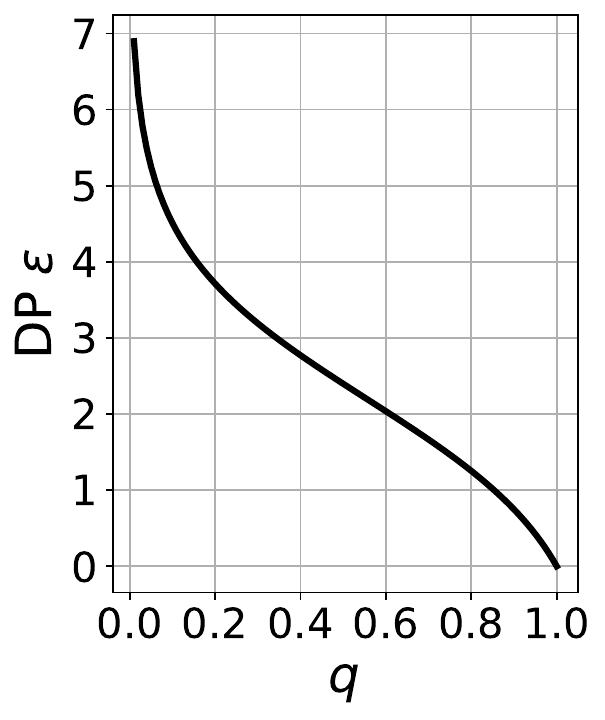}
      \caption{Randomized response}
      \label{fig:synthetic_rr}
    \end{subfigure}
    \begin{subfigure}{.49\textwidth}
      \centering
      \includegraphics[width=.50\linewidth]{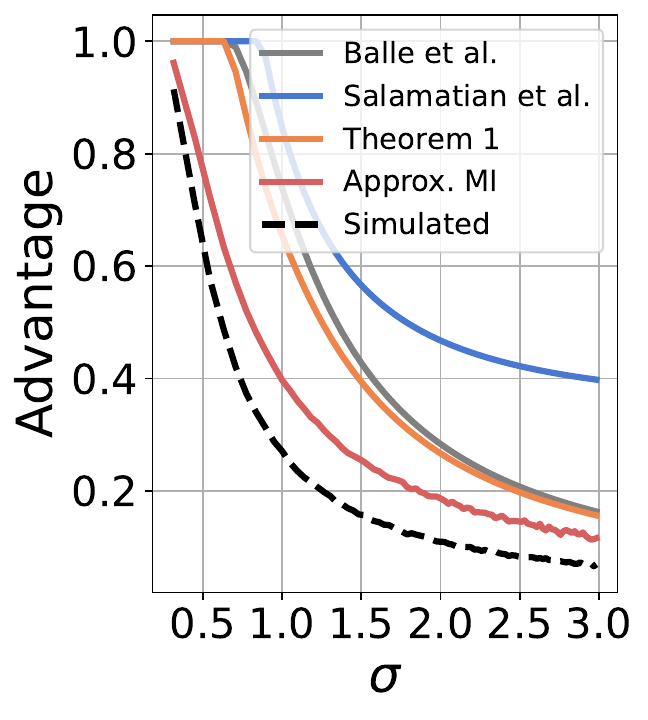} 
      \includegraphics[width=.49\linewidth]{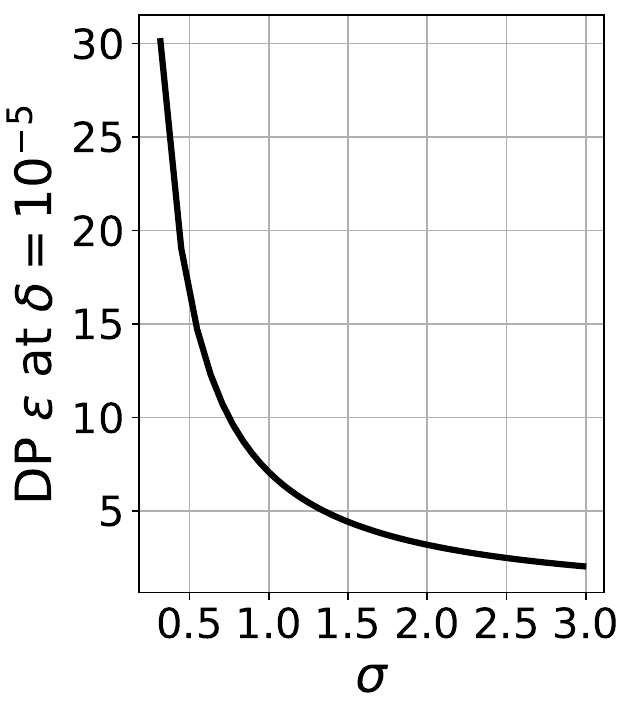}
      \caption{Gaussian mechanism}
      \label{fig:synthetic_gaussian}
    \end{subfigure} 
\caption{Advantage plot for randomized response and the Gaussian mechanism. The advantage bound using Fano's inequality and exact/approximate mutual information closely matches the empirical lower bound. The advantage can be much lower than what the corresponding DP $\epsilon$ suggests in terms of privacy leakage.}
\label{fig:synthetic}
\end{figure*}

\subsection{Exact Computation of $I(X;Y)$}
\label{sec:exact_mi}

For certain DP primitives, the mutual information $I(X;Y)$ can be computed either exactly or estimated numerically.

\paragraph{Randomized response.} The (generalized) randomized response (RR) mechanism is defined as follows:
\begin{equation}
    \label{eq:randomized_response}
    Y|X =
    \begin{cases}
    X & \text{w.p. } 1-q, \\
    \text{Unif}(\{1,\ldots,M\}) & \text{w.p. } q
    \end{cases}
\end{equation}
where $q \in [0,1]$.
%It can be shown that the mechanism is $\epsilon$-DP with $\epsilon = \log \left( \frac{1-q+q/M}{q/M} \right)$.
To derive $I(X;Y)$, first note that $\mathbb{P}(X=x, Y=y) = P(X=x) P(Y=y | X=x)$ and $\mathbb{P}(Y=y) = \sum_x \mathbb{P}(X=x, Y=y)$ can be easily computed given $q$, $M$ and the sampling probability vector $\bp$. We can then numerically compute:
\begin{equation}
    \label{eq:randomized_response_mi}
    \resizebox{\hsize}{!}{$
    I(X;Y) = \sum_x \sum_y \mathbb{P}(X=x, Y=y) \log \frac{\mathbb{P}(X=x, Y=y)}{\mathbb{P}(X=x) \mathbb{P}(Y=y)}.
    $}
\end{equation}

%\paragraph{Tightness of Fano's inequality.} For the randomized response mechanism, one can show that the \emph{maximum a posteriori} (MAP) adversary attains equality in Fano's inequality when the sampling probability for $X$ is uniform, \emph{i.e.}, $\bp_m = 1/M$ for all $m$. This can be done by deriving the posterior distribution for $X$ given any $Y=y$, which we include in \autoref{sec:method_details} for completeness.

\paragraph{Gaussian mechanism.} The Gaussian mechanism is a common primitive in the design of more complex private mechanisms such as DP-SGD~\citep{abadi2016deep}. Here, the mechanism first maps each private record $\bz_m$ to some encoding $\be_m \in \mathbb{R}^d$ and outputs:
\begin{equation}
    \label{eq:gaussian_mechanism}
    (Y|X=m) = \be_m + \mathcal{N}(0, \sigma^2 \mathbf{I}_{d \times d}).
\end{equation}
As a result, $Y$ is a Gaussian mixture with mixture probabilities $\bp$.
For example, in output perturbation~\citep{chaudhuri2011differentially}, $\be_m$ is the model obtained when training a convex model on $\mathcal{D} \cup \{\bz_m\}$; in DP-SGD, $\be_m$ is the gradient of the training loss for sample $\bz_m$. For our purpose, it is equivalent to treat $\be_m$ as the private record instead of $\bz_m$ since the mapping can be easily inverted~\citep{zhu2019deep}.

%It is noteworthy that if we weaken \autoref{thm:mutual_info_gaussian} by substituting $\bp_m \geq \bp_m \exp(-\Delta^2/2\sigma^2)$ into \autoref{eq:mutual_info_gaussian}, we get
%\begin{equation*}
%    I(X;Y) \leq -\sum_m \bp_m \log\bigg(\bp_m + \exp \left( \frac{-\Delta^2}{2\sigma^2} \right) - \bp_m \bigg) = \Delta^2 / 2 \sigma^2,
%\end{equation*}
%which is exactly the $\alpha=1$ R\'{e}nyi divergence bound for the Gaussian mechanism, validating \autoref{thm:renyi_bound}.

One can show that (see \autoref{eq:mutual_information} in \autoref{sec:proofs}):
\begin{equation*}
    I(X;Y) = \sum\nolimits_x \mathbb{P}(X=x) D_1((Y|X=x) || Y),
\end{equation*}
where $D_1$ denotes KL divergence. For the Gaussian mechanism, each term in the sum is the KL divergence between a Gaussian distribution and a mixture of Gaussians. Although this quantity is difficult to bound in closed form, it can be approximated via Monte-Carlo simulation by drawing $k \sim \text{Categorical}(\mathbf{p})$ and computing the KL divergence $D_1((Y|X=k) || Y)$. Taking expectation over multiple draws gives an unbiased estimate of $I(X;Y)$.

\section{Experiment}
\label{sec:experiment}

We perform a series of experiments to show that our Fano's inequality bound provides meaningful semantic guarantees against attribute inference and data reconstruction attacks.

\subsection{Synthetic Data Experiments}
\label{sec:synthetic}

We first experiment on synthetic data to numerically evaluate our Fano's inequality bound for both randomized response and the Gaussian mechanism.

\paragraph{Data generation.} We generate data with $M=10$ possible values sampled uniformly randomly, \emph{i.e.}, $\bp_m = 1/M$ for $m=1,\ldots,M$. For randomized response, we sample the output value $Y$ according to \autoref{eq:randomized_response} with different values of $q \in [0,1]$. For the Gaussian mechanism, we encode the data as one-hot vectors $\be_1,\ldots,\be_M$ and output $Y$ according to \autoref{eq:gaussian_mechanism} with $\sigma \in (0,3]$.

\paragraph{Advantage bounds.} We evaluate our Fano's inequality bound using both the general mutual information bound from RDP (\autoref{thm:renyi_bound}) as well as either exact or approximate computation in \autoref{sec:exact_mi}. As baselines, we evaluate the Fano's inequality bound from \citet{salamatian2020privacy} and the $(\eta,\gamma)$-ReRo bound from \citet{balle2022reconstructing}, which upper bounds the adversary's success probability as: $\mathbb{P}(E=0) \leq (e^\epsilon / M)^{\frac{\alpha-1}{\alpha}}$ if $\calM$ is $(\alpha,\epsilon)$-RDP. We minimize this upper bound over $\alpha$ and convert it to an advantage bound for comparison.

\paragraph{Empirical lower bound.} To compute an empirical lower bound for the advantage, we simulate a data reconstruction adversary using the \emph{maximum a posteriori} (MAP) estimator $\hat{X}(Y)$, which returns the most likely value for $X$ given $Y$ using Bayes' rule. Each experiment then consists of:

\textbf{1.} Sampling $X$ from $1,\ldots,M$ uniformly; \textbf{2.} Using either randomized response or the Gaussian mechanism to sample $Y$; \textbf{3.} Computing the MAP estimate $\hat{X}(Y)$ and comparing it to $X$. We repeat this experiment $100,000$ times to compute an empirical lower bound for the advantage.

\paragraph{Result.} \autoref{fig:synthetic} plots advantage as a function of the sampling probability $q$ for randomized response and noise standard deviation $\sigma$ for the Gaussian mechanism. For randomized response (\autoref{fig:synthetic_rr}), a higher $q$ increases privacy and is reflected in the empirical lower bound (dashed line) as the adversary's advantage decreases to $0$ when $q=1$. The corresponding DP $\epsilon$ for different values of $q$ are shown in the right plot.
%Both the \citet{balle2022reconstructing} bound (gray) and the advantage bound using \autoref{thm:renyi_bound} (blue) rely on RDP accounting.
The bound from \autoref{thm:renyi_bound} is consistently lower than both baseline bounds, although all three bounds are relatively loose compared to the empirical lower bound. In contrast, using Fano's inequality with exact mutual information gives a \emph{tight} advantage upper bound that matches the empirical lower bound. This is to be expected as it can be shown that the MAP estimator attains equality in Fano's inequality under the uniform input distribution; see \autoref{sec:method_details}.
%This validates the derivation in \autoref{sec:randomized_response} showing that the MAP estimator attains equality in Fano's inequality.

For the Gaussian mechanism (\autoref{fig:synthetic_gaussian}), a higher $\sigma$ introduces more randomness and increases privacy, and thus both the empirical lower bound as well as the upper bounds decrease as $\sigma$ increases. Similar to the randomized response result, the baseline bounds are strictly dominated by \autoref{thm:renyi_bound}. Using Monte-Carlo to approximate mutual information gives the tightest advantage bound that closely matches the behavior of the empirical lower bound. These results demonstrate the power and tightness of the Fano's inequality analysis, and suggests that tight mutual information privacy accounting can grant strong semantic privacy guarantees for relatively low amounts of DP noise.

\begin{figure*}[t]
    \centering
    \begin{subfigure}{.32\textwidth}
      \centering
      \includegraphics[width=\linewidth]{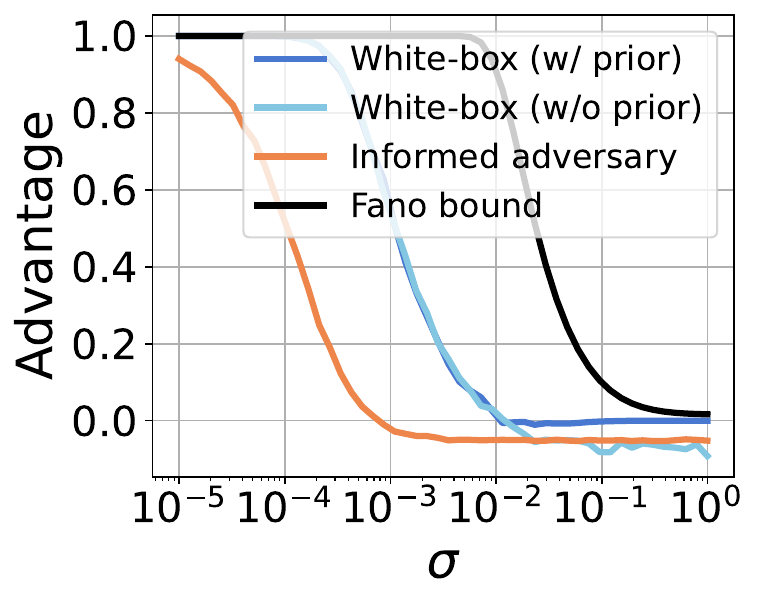}  
      \vspace{-3ex}
      \caption{Advantage vs. $\sigma$}
      \label{fig:attribute_advantage}
    \end{subfigure}
    \begin{subfigure}{.32\textwidth}
      \centering
      \includegraphics[width=\linewidth]{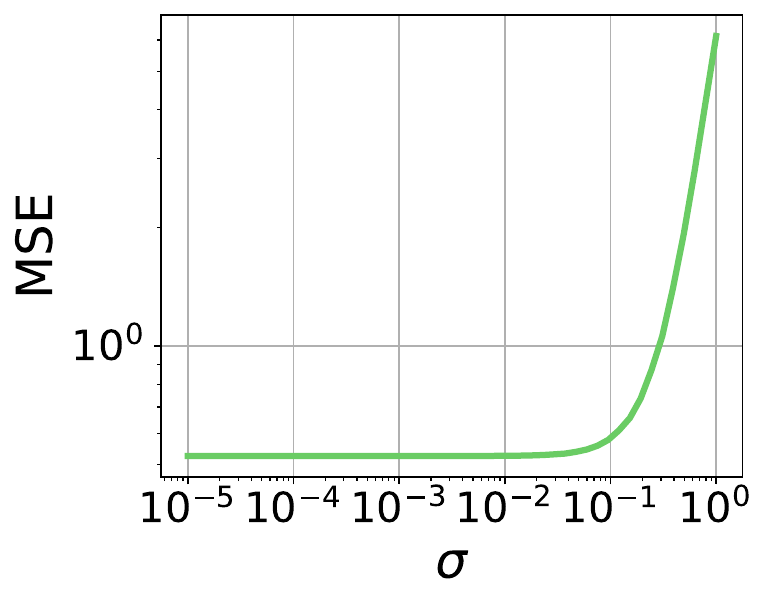}  
      \vspace{-3ex}
      \caption{Test MSE vs. $\sigma$}
      \label{fig:attribute_accuracy}
    \end{subfigure} 
    \begin{subfigure}{.33\textwidth}
      \centering
      \includegraphics[width=\linewidth]{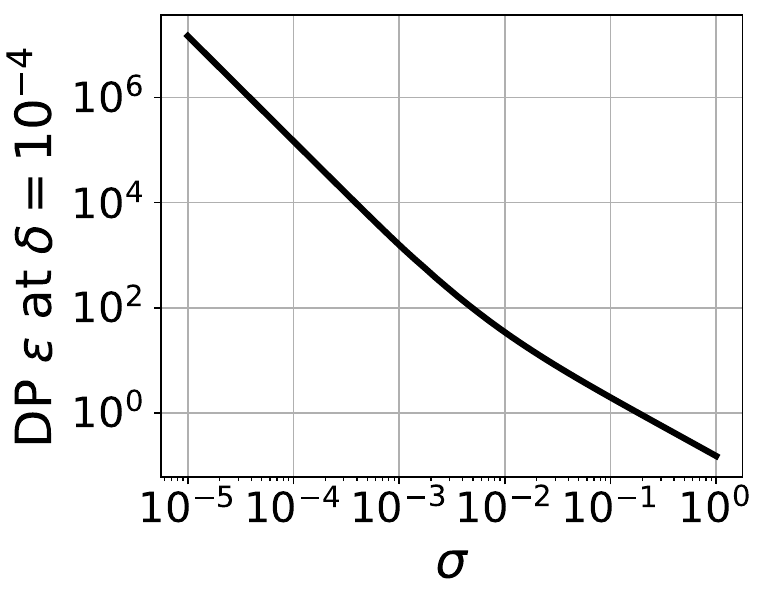}  
      \vspace{-3ex}
      \caption{DP $\epsilon$ vs. $\sigma$}
      \label{fig:attribute_epsilon}
    \end{subfigure} 
\caption{Plots for the IWPC attribute inference attack experiment. The model is trained using output perturbation by adding Gaussian noise with standard deviation $\sigma$ to the loss minimizer. At $\sigma=0.1$, advantage (left) is relatively low when measured according to both the Fano bound and the empirical attacks, while the model's test MSE (middle) increases only moderately. In comparison, the DP $\epsilon$ (right) remains relatively high, overestimating privacy leakage.} %\maziar{I am not sure how this MSE is computed? Which model are we using to evaluate test accuracy? I am curious because for each sample, we have 3 corresponding models? And we do this for each sample.}
\label{fig:attribute_inference}
\end{figure*}

\subsection{Attribute Inference for Pharmacogenetics Modeling}
\label{sec:attribute_inference}

Next, we evaluate our advantage bound for privately trained pharmacogenetics models studied in \citet{fredrikson2014privacy} and show that it gives meaningful protection against attribute inference attacks.

\paragraph{Dataset and model.} The IWPC dataset~\citep{international2009estimation} contains data for clinical trial subjects, with the goal of training an ML model to predict the stable dosage of warfarin given the subjects' attributes. One particularly privacy-sensitive attribute is the VKORC1 gene type, which can be one of three values: \texttt{CC}, \texttt{CT} or \texttt{TT}. We train an $L_2$-regularized linear regression model on the attributes to predict the warfarin dosage (\emph{i.e.}, target label) according to \cite{hannun2021measuring}.

\paragraph{DP mechanism.} We employ the output perturbation mechanism~\citep{chaudhuri2011differentially}, which adds Gaussian noise to the model parameters to guarantee differential privacy. For each sample $j$, let $\calD_{-j}$ be the data subset containing all but the $j$-th sample $\bz^{(j)}$. For the privacy analysis, we train three different models on datasets $\calD_{-j} \cup \{\bz_{m}^{(j)}\}$ for $m=1,2,3$, where $\bz_m^{(j)}$ represents the $j$-th sample with its VKORC1 gene set to one of the three values \texttt{CC}, \texttt{CT} or \texttt{TT}. We can then treat the three models' parameters as encodings of the private gene types under our data reconstruction game formulation, and the output perturbation mechanism is equivalent to the Gaussian mechanism applied to these encodings. To compute the advantage bound, we use a specialized mutual information bound in \autoref{thm:mutual_info_gaussian} in \autoref{sec:proofs} instead of the generic RDP bound in \autoref{thm:renyi_bound}.

\paragraph{Attacks.} We consider attribute inference attacks for inferring the VKORC1 gene type from the trained model. Following \citet{hannun2021measuring}, we implement a white-box attack with full knowledge of the data subset $\calD_{-j}$ and all attributes of sample $\bz^{(j)}$ except for the VKORC1 gene type. The attack essentially implements a maximum likelihood estimator, and we augment it with prior knowledge of the marginal distribution of the three gene types to produce a Bayesian variant. We also implement the Informed Adversary attack from \citet{balle2022reconstructing}, which only has access to $\calD_{-j}$ for predicting the VKORC1 gene type of $\bz^{(j)}$. Details for the attacks are given in \autoref{sec:method_details}.

\paragraph{Result.} \autoref{fig:attribute_advantage} shows the adversary's advantage as a function of $\sigma$, which shows a decreasing trend as expected. The white-box attack with prior knowledge of the marginal VKORC1 gene distribution (dark blue) attains the highest advantage. Notably, its advantage converges to 0 as $\sigma \rightarrow \infty$, whereas the white-box attack without prior (light blue) and the Informed Adversary attack (orange) have negative advantage as $\sigma \rightarrow \infty$. This is because the VKORC1 gene's marginal distribution is non-uniform, with probabilities $0.367, 0.339, 0.294$ for \texttt{CC}, \texttt{CT} and \texttt{TT}, respectively. Since advantage as defined in \autoref{eq:advantage} uses $p^* = 0.367$ (corresponding to always predicting \texttt{CC}) as the baseline, it is sub-optimal to predict the uniform distribution when the model contains close to no information about $\bz^{(j)}$. %\maziar{But you can always use the data $\mathcal{D}_{-j}$ to estimate the prior. Should we at least have these maybe in the appendix?}

The advantage bound using Fano's inequality (black line) strictly dominates the advantage of the three attacks. In particular, at $\sigma=0.1$ the advantage bound predicts close to $0.1$ advantage, which can be considered relatively low risk. Moreover, the test mean squared error (MSE) in \autoref{fig:attribute_accuracy} shows that model utility remains satisfactory when $\sigma=0.1$. At the same time, the corresponding DP $\epsilon$ in \autoref{fig:attribute_epsilon} paints a pessimistic picture, with $\epsilon \approx 2$ at $\sigma=0.1$. Our analysis suggests that $\sigma=0.1$ can be a reasonable value for the output perturbation mechanism applied to this dataset for optimal privacy-utility trade-off. %\maziar{Do we need to have the ReRo baseline here as well?}

\subsection{Language Model Training with Canaries}
\label{sec:language_model}

Finally, we consider a language model canary extraction attack similar to the one performed in \citet{carlini2019secret}.

\paragraph{Setup.}
We experiment using a pre-trained GPT-2~\citep{radford2019language} model---a causal language model for predicting the next token given a context token sequence. The vocabulary consists of $M=50,256$ tokens.
We fine-tune the model using DP-SGD~\citep{abadi2016deep} on a single \emph{canary sequence} of the form ``\texttt{John Smith's credit card number is X}", with \texttt{X} being a token chosen uniformly randomly from the vocabulary.
For the DP-SGD hyperparameters, we use a clipping norm of $C=1$ and vary the number of fine-tuning steps $T$ and the noise multiplier $\sigma$. For any $\alpha \geq 1$, one can show that DP-SGD is $(\alpha,\epsilon)$-RDP with $\epsilon=\alpha T / 2 \sigma^2$, and then use Fano's inequality with \autoref{thm:renyi_bound} to obtain the advantage upper bound.

\paragraph{Attack.}
Since the model is fine-tuned to predict a single secret token \texttt{X} in the canary sequence, we can perform an extraction attack that compares the log-likelihood\footnote{We also experimented using the likelihood and found that the attack performs significantly worse.} of every token \texttt{X} before and after fine-tuning. The adversary then predicts the token with the largest difference. Note that such a differencing attack has been used in state-of-the-art membership inference attacks such as \citet{ye2021enhanced, watson2021importance, carlini2022membership}. Thus, each experiment run consists of sampling \texttt{X} uniformly from $M=50,256$ tokens, fine-tuning the GPT-2 model using DP-SGD on the canary sequence, and then running the extraction attack to recover \texttt{X}. We perform this experiment $1000$ times to compute the adversary's advantage empirically.

\begin{figure}[t]
    \centering
    \includegraphics[width=\linewidth]{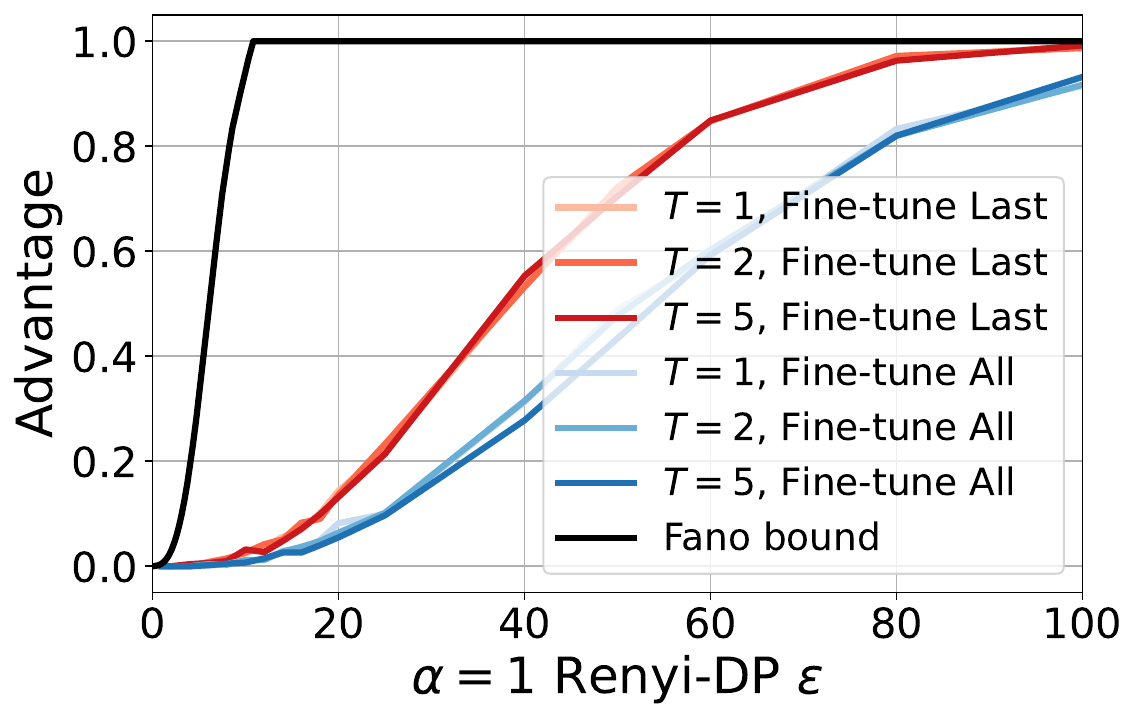}
    \caption{Result for the language model canary extraction attack. The model is fine-tuned on the canary sequence for $T$ steps using DP-SGD.
    %The advantage upper bound using Fano's inequality gives non-trivial semantic guarantee against reconstruction attacks even for relatively large values of RDP parameter $\epsilon$.
    See text for details.}
\label{fig:nlp}
\end{figure}

%As an attack, we evaluate two strategies. 
%We compare the predictions of the pretrained GPT2 with and without fine-tuning, and take the token that is the most over-predicted by the fine-tuned model.
%We use either the raw probability scores or their logarithm to compare across models. 

\paragraph{Result.} \autoref{fig:nlp} shows the adversary's advantage as a function of the $\alpha=1$ RDP parameter $\epsilon=2T/\sigma^2$, where $T \in \{1,2,5\}$ is the number of fine-tuning steps. We include two variants of the fine-tuning method, either fine-tune only the last layer (red) or fine-tune all layers (blue). As expected, the adversary's advantage increases monotonically as $\epsilon$ increases, while the advantage bound using Fano's inequality strictly dominates the advantage of simulated extraction attacks. Although the advantage bound overestimates the empirically computed advantage, it can still serve as a useful guideline for selecting the privacy parameter $\epsilon$. For instance, at $\epsilon=6$, the advantage upper bound is close to $0.5$, which is certainly a non-trivial level of privacy protection for a relatively large value of $\epsilon$. We suspect that the large gap is a result of upper bounding mutual information using the generic RDP bound in \autoref{thm:renyi_bound} and is an inherent limitation of using RDP accounting. A closer privacy analysis of the Gaussian mechanism directly in terms of mutual information may be able to close this gap significantly.
\section{Conclusion}
\label{sec:conclusion}

We presented a rigorous formulation of data reconstruction attacks and analyzed the privacy leakage of DP mechanisms using Fano's inequality. Our analysis gives a numerical method for upper bounding the advantage of a reconstruction adversary, which we show empirically can be a strong indicator for the mechanism's actual privacy protection against privacy attacks. Consequently, we advocate for the use of our bounds both for interpreting the DP parameter $\epsilon$ and as a guideline for selecting it in practice. %Future research may be able to further tighten the analysis, especially for more complex mechanisms involving composition, to build a comprehensive privacy analysis framework with rigorous semantic guarantees.

%\paragraph{Societal impact.} Privacy is an important consideration in order for ML to gain social acceptance. While the framework of differential privacy has enabled privacy-preserving ML on sensitive data, at the moment there is very limited understanding of the privacy semantics of DP. Such semantics is undoubtedly important when communicating the privacy guarantee to non-experts, as well as potentially inform policy decisions and influence regulation on what is an acceptable level of privacy guarantee for sufficient protection. Our work aims at providing a better semantic understanding of DP, which can serve as a key step towards its wide-spread social acceptance.

\bibliography{main}
\bibliographystyle{icml2023}

\newpage
\appendix
\onecolumn
\setcounter{theorem}{0}
\section{Method Details}
\label{sec:method_details}

\subsection{Generalized Fano's Inequality Analysis}
\label{sec:generalized_fano}

We derive a generalized version of the Fano's inequality analysis in \autoref{sec:fano} using Arimoto information. We first define a few key quantities in information theory below.

\paragraph{Measures of information.} In information theory, \emph{entropy} measures the amount of information contained in a random variable. For a discrete random variable $X$, the entropy of $X$ is defined as
\begin{equation}
    \label{eq:entropy}
    H(X) = -\sum_x \mathbb{P}(X=x) \log \mathbb{P}(X=x).
\end{equation}
A closely related notion is \emph{conditional entropy} $H(X|Y)$. For two random variables $X$ and $Y$, $H(X|Y)$ measures the amount of information contained in $X$ after observing $Y$, and is defined as
\begin{equation}
    \label{eq:cond_entropy}
    H(X|Y) = \sum_y \mathbb{P}(Y=y) H(X|Y=y),
\end{equation}
where $H(X|Y=y)$ denotes entropy of the conditional distribution $X|Y=y$.
In essence, if $X$ and $Y$ are highly dependent then $H(X|Y) \approx 0$, which suggests that there is very little entropy left in $X$ after observing $Y$. On the other hand, if $X$ and $Y$ are independent then $H(X|Y) = H(X)$, meaning that observing $Y$ gives no additional information about $Y$. Thus, the difference
\begin{equation}
    \label{eq:mutual_info}
    I(X;Y) := H(X) - H(X|Y),
\end{equation}
known as \emph{mutual information}, is a non-negative quantity that measures the amount of dependence between $X$ and $Y$.

\paragraph{Generalization using R\'{e}nyi entropy.} The three measures of information defined above can be generalized. We start from the generalization of entropy to \emph{R\'{e}nyi entropy of order $\alpha$}~\citep{renyi1961measures}, denoted $H_\alpha$, given by:
\begin{equation}
    \label{eq:renyi_entropy}
    H_\alpha(X) = \frac{1}{1-\alpha} \log \sum_x \mathbb{P}(X=x)^\alpha
\end{equation}
for $\alpha > 1$. Notably, it can be shown that this definition recovers the standard entropy $H(X)$ when $\alpha \rightarrow 1$. The corresponding generalization of conditional entropy is the so-called \emph{Arimoto-R\'{e}nyi conditional entropy}~\citep{arimoto1977information}, defined as
\begin{equation}
    \label{eq:renyi_cond_entropy}
    H_\alpha(X|Y) = \frac{\alpha}{1-\alpha} \log \mathbb{P}(Y=y) \exp \left( \frac{1-\alpha}{\alpha} H_\alpha(X|Y=y) \right).
\end{equation}
The generalization of mutual information $I_\alpha(X;Y)$, known as \emph{Arimoto information}~\citep{verdu2015alpha}, is the difference $H_\alpha(X) - H_\alpha(X|Y)$, and has the following form:
\begin{equation}
    \label{eq:arimoto}
    I_\alpha(X;Y) = \frac{\alpha}{\alpha-1} \log \sum_y \left( \sum_x \mathbb{P}(X_\alpha=x) \mathbb{P}(Y=y | X=x)^\alpha \right)^{1/\alpha},
\end{equation}
where $X_\alpha$ denotes the $\alpha$-scaled distribution with $\mathbb{P}(X_\alpha=x) \propto \mathbb{P}(X=x)^\alpha$. Similar to R\'{e}nyi entropy, Arimoto-R\'{e}nyi conditional entropy and Arimoto information both converge to the corresponding conditional entropy and mutual information as $\alpha \rightarrow 1$.

\paragraph{Generalized Fano's inequality.} Fano's inequality can be generalized using an analogue of conditional entropy known as \emph{Arimoto-R\'{e}nyi conditional entropy}~\citep{arimoto1977information}. The generalized version is given in \cite{sason2016f}, which states:
\begin{equation}
    \label{eq:generalized_fano}
    H_\alpha(X | Y) \leq \log M - D_\alpha(\text{Bernoulli}(t) || \text{Bernoulli}(1-1/M)),
\end{equation}
where $\text{Bernoulli}(t)$ denotes the Bernoulli distribution with success probability $t$.
It can be shown that $H_\alpha(X|Y) = H_\alpha(X) - I_\alpha(X;Y)$. $H_\alpha(X)$ can be computed directly using \autoref{eq:renyi_entropy}, while $I_\alpha(X;Y)$ can be upper bounded using \autoref{thm:renyi_bound}. Thus, we can lower bound $H_\alpha(X|Y)$ and establish an inequality for $t$ using the generalized Fano's inequality. We can then apply the binary search method in Algorithm \ref{alg:advantage_bound} to upper bound the advantage for a given $\alpha$. Algorithm \ref{alg:generalized_advantage_bound} details the computation in pseudo-code.

\begin{algorithm}[h]
\caption{Computing the advantage bound using generalized Fano's inequality.}
\label{alg:generalized_advantage_bound}
\begin{algorithmic}[1]
\STATE \textbf{Input}: $\alpha \geq 1$, upper bound $I_\alpha(X;Y) \leq \mu$, number of distinct attributes $M$, sampling probability vector $\bp$.
\STATE $p^* \gets \max_m \bp_m$
\IF{$\mu \geq H_\alpha(\bp)$}
\STATE \textbf{return} $1$
\ELSE
\STATE Define $f_\mu(t) = H_\alpha(\bp) - \mu - \log M + D_\alpha(\text{Bernoulli}(t) || \text{Bernoulli}(1-1/M))$.
\STATE \textbf{assert} $f_\mu(0) > 0$ and $f_\mu(1-1/M) \leq 0$
\STATE Use binary search to find $t^* = \min \{t \in [0,1-1/M]: f_\mu(t) \leq 0\}$.
\STATE \textbf{return} $(1-t^*-p^*) / (1-p^*)$
\ENDIF
\end{algorithmic}
\end{algorithm}

\subsection{Tightness of Fano's Inequality for Randomized Response}
\label{sec:rr_tightness}

Recall the generalized randomized response mechanism in \autoref{sec:exact_mi}. We show that the Fano's inequality analysis is tight when the input is sampled uniformly randomly. For any $Y=y$, we can derive the posterior distribution for $X$ using Bayes rule:
\begin{equation*}
    \mathbb{P}(X=x|Y=y) = \mathbb{P}(Y=y|X=x) \mathbb{P}(X=i) / \mathbb{P}(Y=y) =
    \begin{cases}
    1-q+q/M & \text{if } x=y,\\
    q/M & \text{otherwise.}
    \end{cases}
\end{equation*}
Suppose $q < 1$ so that $1-q+q/M > q/M$, then the MAP adversary's estimate is $\hat{X}(Y) = Y$, with error probability $\mathbb{P}(E=1) = q - q/M$. For any $y$, we can derive the conditional entropy $H(X|Y=y)$ in closed form:
\begin{align*}
    H(X|Y=y) &= \left( -(1-q+q/M) \log(1-q+q/M) - \sum_{x \neq y} \frac{q}{M} \log \frac{q}{M} \right) \\
    &= \left( -(1-q+q/M) \log(1-q+q/M) - (q-q/M) \log \frac{q(M-1)}{M(M-1)} \right) \\
    &= (-(1-q+q/M) \log(1-q+q/M) - (q-q/M) \log (q-q/M) + (q-q/M) \log (M-1)) \\
    &= H(E) + \mathbb{P}(E=1) \log(M-1).
\end{align*}
Multiplying both sides by $\mathbb{P}(Y=y)$ and taking sum over $y$ gives $H(X|Y) = H(E) + \mathbb{P}(E=1) \log(M-1)$, which attains equality in \autoref{eq:fano}.

\subsection{Attribute Inference Attack Details}
\label{sec:attribute_inference_details}

We give details for the white-box attribute inference attack in \autoref{sec:attribute_inference}. Fix sample $j$, and for $m=1,2,3$ let $\bw_m$ be the parameter vector of the model trained on $\calD_{-j} \cup \{\bz_m^{(j)}\}$. Suppose the output perturbation mechanism returns the parameter vector $\bw$. Since the mechanism obtained $\bw$ by perturbing one of $\bw_1,\bw_2,\bw_3$ with Gaussian noise, we can evaluate the likelihood of $\bw$ under all three choices and maximize it to determine $X$:
\begin{equation*}
    \log \mathbb{P}(X=m|Y=\bw) = -\|\bw - \bw_m \|_2^2 / 2 \sigma^2 + \text{constant}.
\end{equation*}
This is possible since the adversary has access to $\calD_{-j}$ and all attributes of $\bz^{(j)}$ except for the VKORC1 gene, and hence can enumerate all possible values for $\bz_i^{(j)}$ to obtain $\bw_1,\bw_2,\bw_3$. To derive the MAP estimator, we apply Bayes' rule:
\begin{align*}
    \mathbb{P}(X=m|Y=\bw) &\propto \mathbb{P}(Y=\bw|X=m) \mathbb{P}(X=m) = \mathbb{P}(Y=\bw|X=m) \bp_m \\
    \Rightarrow \log \mathbb{P}(X=m|Y=\bw) &= -\|\bw - \bw_m \|_2^2 / 2 \sigma^2 + \log \bp_m  + \text{constant},
\end{align*}
where $\bp = (0.367, 0.339, 0.294)$ is the marginal distribution of the VKORC1 gene.

\section{Proofs}
\label{sec:proofs}

\begin{theorem}
For any $\alpha \geq 1$, if $\mathcal{M}$ is $(\alpha,\epsilon)$-RDP then $I_\alpha(X;Y) \leq \epsilon$.
\end{theorem}

\begin{proof}
We first prove this for $\alpha=1$. By definition:
\begin{align}
    I(X;Y) &= \sum_x \sum_y \mathbb{P}(X=x, Y=y) \log \frac{\mathbb{P}(X=x, Y=y)}{\mathbb{P}(X=x) \mathbb{P}(Y=y)} \nonumber \\
    &= \sum_x \mathbb{P}(X=x) \sum_y \mathbb{P}(Y=y | X=x) \log \frac{\mathbb{P}(Y=y | X=x)}{\mathbb{P}(Y=y)} \nonumber \\
    &= \sum_x \mathbb{P}(X=x) D_1((Y | X=x) || Y), \label{eq:mutual_information}
\end{align}
where $D_1$ denotes KL divergence, which is equivalent to the order $\alpha=1$ R\'{e}nyi divergence. By convexity,
\begin{align*}
    D_1((Y | X=x) || Y) &\leq \sum_{x'} \mathbb{P}(X=x') D_1((Y | X=x) || (Y | X=x')) \\
    &\leq \sum_{x'} \mathbb{P}(X=x') \epsilon \quad \quad \text{since $\mathcal{M}$ is $(1,\epsilon)$-RDP} \\
    &= \epsilon.
\end{align*}
Plugging back into \autoref{eq:mutual_information} gives the desired result. For $\alpha > 1$, let $X_\alpha$ denote the $\alpha$-scaled distribution in the definition of Arimoto information, \emph{i.e.}, $\mathbb{P}(X_\alpha = x) \propto \mathbb{P}(X=x)^\alpha$. Then:
\begin{align*}
    I_\alpha(X;Y) &= \frac{\alpha}{\alpha-1} \log \sum_y \left( \sum_x \mathbb{P}(X_\alpha=x) \mathbb{P}(Y=y | X=x)^\alpha \right)^{1/\alpha} \\
    &= \frac{\alpha}{\alpha-1} \log \sum_y \mathbb{P}(Y=y) \left( \sum_x \mathbb{P}(X_\alpha=x) \frac{\mathbb{P}(Y=y | X=x)^\alpha}{\mathbb{P}(Y=y)^\alpha} \right)^{1/\alpha} \\
    &\leq \frac{1}{\alpha-1} \log \sum_y \mathbb{P}(Y=y) \left( \sum_x \mathbb{P}(X_\alpha=x) \frac{\mathbb{P}(Y=y | X=x)^\alpha}{\mathbb{P}(Y=y)^\alpha} \right) \quad \text{by convexity of $z \mapsto z^\alpha$} \\
    &\leq \max_x \frac{1}{\alpha-1} \log \left( \sum_y \mathbb{P}(Y=y) \frac{\mathbb{P}(Y=y | X=x)^\alpha}{\mathbb{P}(Y=y)^\alpha} \right) \quad \text{by monotonicity of $z \mapsto \frac{1}{\alpha-1} \log z$} \\
    &= \max_x D_\alpha((Y|X=x) || Y).
\end{align*}
The proof follows by quasi-convexity of R\'{e}nyi divergence and the fact that $\calM$ is $(\alpha,\epsilon)$-RDP.
\end{proof}

\paragraph{Mutual information bound for Gaussian mechanism.} The Gaussian mechanism satisfies $(\alpha,\epsilon)$-RDP with $\epsilon = \alpha \Delta^2 / 2 \sigma^2$ when the encoding function has $L_2$-sensitivity $\Delta$~\citep{mironov2017renyi}. This readily gives a bound for $I(X;Y)$ using \autoref{thm:renyi_bound} by setting $\alpha=1$. Below we present a tighter bound for the mutual information of Gaussian mechanisms.

\begin{theorem}
\label{thm:mutual_info_gaussian}
Let $\be_m$ be the encoding for the private record $\bz_m$ for $m=1,\ldots,M$, and let $\Delta = \max_{m,l} \| \be_m - \be_l \|_2$ be the $L_2$-sensitivity of the encoding function. Then the Gaussian mechanism (\emph{cf.} \autoref{eq:gaussian_mechanism}) satisfies:
\begin{equation}
    \label{eq:mutual_info_gaussian}
    I(X;Y) \leq -\sum_m \bp_m \log\bigg(\bp_m +(1-\bp_m) \exp \left( \frac{-\Delta^2}{2\sigma^2} \right)\bigg).
\end{equation}
\end{theorem}

\begin{proof}
First note that $I(X;Y) = H(Y)-H(Y|X)$ with $H(Y|X) = \frac{M}{2}\log(2\pi e \sigma^2)$.
By Section 3.1 of \cite{nielsen2017w}, because $Y$ is a mixture of Gaussians we have the following bound for $H(Y)$:
\begin{align*}
    H(Y) &\leq \sum_m \bp_m H\bigg(\mathcal{N}(\be_m, \sigma^2\mathbf{I}_{dxd})\bigg) - \sum_m \bp_m \log\left(\sum_l \bp_l \exp\bigg(-D_1\bigg(\mathcal{N}(\be_m, \sigma^2\mathbf{I}_{dxd}), \mathcal{N}(\be_l, \sigma^2\mathbf{I}_{dxd})\bigg)\bigg)\right) \\
    &\leq \frac{M}{2}\log\bigg(2\pi e \sigma^2\bigg) -\sum_m \bp_m \log\bigg(\bp_m + (1-\bp_m) \exp\left(\frac{-\Delta^2}{2\sigma^2}\right)\bigg).
\end{align*}
Combining the two quantities gives the desired result.
\end{proof}
\end{document}